\documentclass{article}

\usepackage{arxiv}
\providecommand{\undertitle}[1]{#1}   
\providecommand{\headeright}{}        

\usepackage{graphicx}
\usepackage{natbib}
\usepackage{doi}

\usepackage[utf8]{inputenc} 
\usepackage[T1]{fontenc}    
\usepackage{hyperref}       
\usepackage{url}            
\usepackage{booktabs}       
\usepackage{amsfonts}       
\usepackage{nicefrac}       
\usepackage{microtype}      
\usepackage{xcolor}         
\usepackage{amsthm}
\usepackage{amsmath}
\usepackage{amssymb}
\usepackage{algorithmic}
\usepackage{algorithm}
\newtheorem{definition}{Definition}
\newtheorem{assumption}{Assumption}

\usepackage{graphicx}
\usepackage{thmtools} 
\usepackage{thm-restate}

\title{Networked Restless Multi-Arm Bandits with Reinforcement Learning}

\author{%
  Hanmo Zhang \\
  Carnegie Mellon University\\
  Pittsburgh, PA 15213 \\
  \texttt{hanmozha@andrew.cmu.edu} \\
  \And
  Zenghui Sun \\
  Georgia Institute of Technology \\
  Atlanta, GA 30339 \\
  \texttt{zsun424@gatech.edu} \\
  \And
  Kai Wang \\
  Georgia Institute of Technology \\
  Atlanta, GA 30339 \\
  \texttt{kwang692@gatech.edu} \\
}

\begin{document}

\maketitle

\begin{abstract}
Restless Multi-Armed Bandits (RMABs) are a powerful framework for sequential decision-making, widely applied in resource allocation and intervention optimization challenges in public health. However, traditional RMABs assume independence among arms, limiting their ability to account for interactions between individuals that can be common and significant in a real-world environment. This paper introduces Networked RMAB, a novel framework that integrates the RMAB model with the independent cascade model to capture interactions between arms in networked environments.
We define the Bellman equation for networked RMAB and present its computational challenge due to exponentially large action and state spaces. 
To resolve the computational challenge, we establish the submodularity of Bellman equation and apply the hill-climbing algorithm to achieve a $1-\frac{1}{e}$ approximation guarantee in Bellman updates. 
Lastly, we prove that the approximate Bellman updates are guaranteed to converge by a modified contraction analysis. 
We experimentally verify these results by developing an efficient Q-learning algorithm tailored to the networked setting. Experimental results on real-world graph data demonstrate that our Q-learning approach outperforms both $k$-step look-ahead and network-blind approaches, highlighting the importance of capturing and leveraging network effects where they exist.

\end{abstract}
\section{Introduction}
\label{sec:introduction}

Public health challenges such as infectious disease control, vaccination strategies, and chronic illness management require sophisticated sequential decision-making under uncertainty, where timely decisions can significantly impact population health outcomes~\citep{world2019world}. The Restless Multi-Armed Bandit (RMAB) framework has emerged as a powerful tool for addressing such sequential decision-making problems under resource constraints. Prior work has successfully applied variations of RMABs to various public health settings, such as optimizing treatment strategies for infectious diseases~\citep{mate2020collapsing}, designing treatment policies for tuberculosis patients ~\citep{mate2021rmabplanning}, efficient streaming‑patient intervention planning \citep{mate2022streamingrmab}, and fair resource allocation across patient cohorts \citep{li2022fairnessrmab}.

However, a significant limitation of traditional RMAB models is the assumption of independence among arms. In many public health applications, the state of one individual directly affects others due to network effects. For example, in epidemic processes infection propagates along contact networks, so an individual's health status changes the risk faced by their neighbors \citep{pastor2001epidemic,wang2003epidemic,pastor2015epidemic,kiss2017mathematics}. During the COVID‑19 pandemic, the impact of interventions such as vaccination or quarantine depended not only on who was targeted but also on the topology of the underlying interaction graph \citep{world2020world,funk2010modelling}.
Ignoring such dependencies can yield sub‑optimal resource allocation, higher transmission rates, and increased morbidity. To model such interactions, the Independent Cascade (IC) model has been widely used to capture the probabilistic spread of influence through a network~\citep{kempe2003maximizing}. 

Our work introduces the \textbf{Networked Restless Multi-Armed Bandit (NRMAB)} framework, which integrates the RMAB model with the Independent Cascade model to account for network effects. This enables more realistic representations of how interventions on one individual can influence the health states of others. By incorporating network effects, our model allows the action on one arm to influence not only its own state transitions but also those of neighboring arms through cascades. 

We formulate Bellman’s equation for this networked problem and prove that the value function is submodular. Because activation is probabilistic and arms can passively change states, traditional submodularity proofs for independent cascade no longer work \citep{kempe2003maximizing}.We adapt the original proof to this setting, accommodating probabilistic activation and passive state changes. Submodularity in turn unlocks a greedy hill‑climbing action-selection policy for Bellman equation whose return is at least $(1-1/e)$ of the optimal~\citep{nemhauser1978analysis}.

We then establish that the Bellman operator with hill‑climbing action selection is a $\gamma$‑contraction. Because greedy selection can be sub‑optimal, classical contraction proofs that rely on optimal action selection no longer apply. We design an equivalent multi-bellman operator with a meta‑MDP and show this operator contracts under the supremum norm \citep{carvalho2023multibellmanoperatorconvergenceqlearning}. Value iteration converges linearly, and finite‑horizon implementations inherit tight error bounds, ensuring practical algorithms remain stable and sample‑efficient.

Building on this theoretical foundation, we develop a Q-learning algorithm for NRMABs. Our algorithm uses hill-climbing action selection for the Bellman equation to approximate the optimal policy without the need to compute the exact value function, which is computationally infeasible in large networks. We validate our approach through experiments on synthetic networks, demonstrating that our network-aware algorithm outperforms network-blind baselines, including the traditional Whittle Index policy~\citep{whittle1988restless}. These results highlight the importance of capturing network effects in sequential decision-making problems and suggest that NRMABs can provide more effective intervention strategies in public health and other domains where networked interactions are significant.

\section{Related Works}
\label{sec:related_works}
\paragraph{Restless multi-armed bandits}
RMABs, first introduced by \citet{whittle1988restless}, extend the classic Multi-Armed Bandit framework to scenarios where each arm evolves over time regardless of whether it is selected, making a powerful model for decision-making problems in uncertain and evolving environments. Finding optimal policies for RMABs is PSPACE-hard \citep{papadimitriou1999complexity}, leading to the development of various approximation algorithms, such as the Whittle index policy \citep{whittle1988restless}. RMABs have been applied in domains such as machine maintenance~\citep{glazebrook2006index}, healthcare~\citep{mate2020collapsing}, and communication systems~\citep{liu2010indexability}.

\paragraph{Independent Cascade Model}
The Independent Cascade model, introduced by ~\citet{kempe2003maximizing}, captures the probabilistic spread of influence through networks and is a fundamental framework for studying diffusion processes in social networks. In this model, active nodes have a single chance to activate each inactive neighbor with certain probability, modeling phenomena such as information spread and epidemic propagation. Influence maximization -- selecting a set of initial nodes to maximize the expected spread -- is NP-hard but benefits from submodularity, which allows for efficient approximation algorithms with provable guarantees \citep{nemhauser1978analysis}. Submodular function maximization has been extensively studied and applied to various network optimization problems~\citep{leskovec2007cost, chen2010scalable}.

\paragraph{Networked Bandits}
Prior work has explored extending RMABs to account for network effects. \citet{networked2022} introduced a RMAB framework accounting for movement of people between physical locations. 
\citet{herlihy2023networked} incorporated network effect by giving each arm a "message" action that influences the transition probability of neighboring arms. \citet{agarwal2024network} modifies the restless multi-armed bandit problem such that the reward on each arm depends on the actions performed on its neighboring arms. 

These works confirm the value of incorporating graph structure, yet each targets a specific form of coupling. Our NRMAB framework advances this foundation by modeling probabilistic cascades of state transitions and providing a submodular‑greedy RL solution with contraction guarantees, yielding scalable policies for networked health‑intervention problems.

\paragraph{Q-Learning}
Q-learning~\citep{watkins1992q} is a model-free reinforcement learning algorithm that learns an optimal action-selection policy by iteratively updating Q-values based on observed rewards and transitions. The algorithm is well-suited for decision-making in Markov Decision Processes (MDPs) and has been widely applied in domains such as game-playing agents~\citep{mnih2015human}. While tabular Q-learning is effective for small state spaces, it suffers from scalability issues as the state-action space grows. Deep Q Networks (DQNs)~\citep{mnih2015human} address this limitation by approximating the Q-function using deep neural networks, enabling Q-learning to scale to large state spaces. 

Particularly relevant is ~\cite{dai2017learning}, who explored reinforcement learning for combinatorial optimization problems on graphs and demonstrated that graph structures can be leveraged to learn effective heuristics for NP-hard problems~\citep{dai2017learning}. This motivates our approach where we leverage Q-learning to optimize decision-making in dynamic networked environments.

\section{Problem Setting}
\label{sec:Problem Setting}
\subsection{RMAB Problem Formulation}

A RMAB \citep{whittle1988restless}
consists of \(n\) independent arms that evolve in parallel.
At each timestep the controller may activate at most \(k\) arms.
Let $\mathcal{S}=\{0,1\}^{n}$ and $\mathcal{A}=  \{\boldsymbol{a}\in\{0,1\}^{n}:\sum_{i=1}^{n} a_i=k\}$ where \(s_i\in\{0,1\}\) denotes the two‑state status of arm \(i\) and \(a_i\in\{0,1\}\) denotes the two‑action choice. Quality of the action is determined by a \textbf{reward function}, which we define in the next subsection. 

\paragraph{Independent arm transition} 
Given the state $s$ and the action $a$ of arm $v$, the state transitions to the next state $u$ based on the transition probability \( P_v(s, a, u) \). $P_v(s,a,u)$ is a probability distribution over the next states, which is independent for all arms. We write the independent transition of the current state of all nodes $\boldsymbol{s} = [s_v]_{v \in \mathcal{V}}$ by $P(\boldsymbol{u} | \boldsymbol{s}, \boldsymbol{a}) = \prod\nolimits_{v \in \mathcal{V}} P_v(s_v, a_v, u_v)$.

\begin{assumption}
    We assume that active actions yield higher probabilities of beneficial transitions compared to passive actions: $P(s=0, a=1, u=1) \geq P(s=0, a=0, u=1)$ and $P(s=1, a=1, u=1) \geq P(s=1, a=0, u=1)$.
\end{assumption}

This compact MDP representation, standard in modern RMAB surveys (e.g.\ \citealp{nino2023markovian}), underpins the network extensions developed in subsequent sections.

\subsection{Independent cascade} We now add the IC model \citep{kempe2003maximizing}. We connect arms (now called nodes) through undirected edges \( e \in \mathcal{E} \), where each edge has a weight \( 0 < w_e < 1 \) that represents the probability an active node activates its neighbor via a cascade. We use a function $P_G(\boldsymbol{s}' | \boldsymbol{u})$ to denote the probability that the temporary state $\boldsymbol{u}$ cascades to the next state $\boldsymbol{s}' = [s'_v]_{v \in \mathcal{V}}$ through the graph $G$ and the cascade probability of each edge.

\paragraph{Transition kernel}
Coupling the arm‑wise dynamics \(P(\boldsymbol{u}\mid\boldsymbol{s},\boldsymbol{a})\)
from the RMAB with the cascade yields the full MDP kernel
\begin{equation}\label{eq:full_transition}
  P(\boldsymbol{s}'\mid\boldsymbol{s},\boldsymbol{a})
  \;=\;
  \sum\nolimits_{\boldsymbol{u}\in\{0,1\}^{n}}
     P(\boldsymbol{u}\mid\boldsymbol{s},\boldsymbol{a})\;
     P_G(\boldsymbol{s}'\mid\boldsymbol{u}),
\end{equation}
where
\(P(\boldsymbol{u}\mid\boldsymbol{s},\boldsymbol{a})
  =\prod_{v\in\mathcal{V}} P_v(u_v\mid s_v,a_v)\)
is the independent arm transition introduced earlier.

\noindent\textbf{Reward objective}
Our goal is to select the optimal $k$ nodes at each timestep to maximize the cumulative reward over multiple timesteps \( t \). The reward function \( R(s, a) \) is the immediate reward received per step after taking action \( a \) in state \( s \). \( \mathcal{V} \) represents the set of all nodes in the graph, and \( r(v) \) is the value associated with node \( v \) if it is active (\( s(v) = 1 \)), or zero otherwise. Cumulative reward is formalized using the discounted return where \( \gamma \) is the discount factor (\( 0 \leq \gamma < 1 \)) that prioritizes immediate rewards over distant future rewards. 
\begin{align}\label{eq:discounted_return}
    &
    R(\boldsymbol{s}, \boldsymbol{a}) = \sum\nolimits_{v \in \mathcal{V}} r(v),
    &
    \sum\nolimits_{t=0}^{\infty} \gamma^t R(\boldsymbol{s_{t}}, \boldsymbol{a_{t}}),
    &
\end{align}
\subsection{Network RMAB Problem Formulation}
An instance of the network RMAB problem is composed of a graph $G = (\mathcal{V},\mathcal{E})$, where each node $v \in \mathcal{V}$ represents an arm that can transition between different states $s \in \mathcal{S}$.
Each iteration we have a budget constraint on the actions: $\sum_{i \in [n]} a_i \leq k$. We can cast the Networked RMAB as a discounted Markov decision process  
\(\mathcal{M} = (\mathcal{S},\mathcal{A},P,R,\gamma)\).
\begin{figure*}[h!]
    \includegraphics[width=1\linewidth]{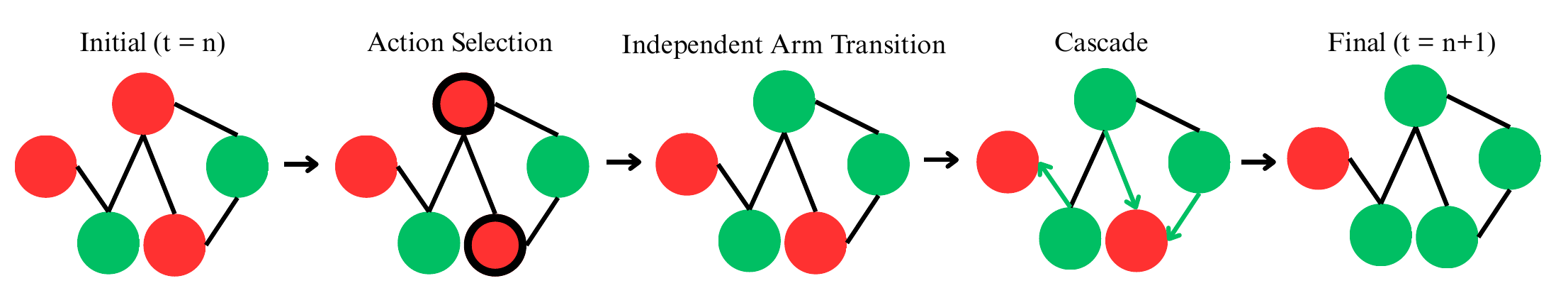}
   \caption{Visual representation of a single Networked RMAB timestep: initial state, selection of \(k\) active actions, independent transitions, cascade propagation, and resulting next state.}
    \label{fig:Visual_problem_formulation}
\end{figure*}
\section{Methodology}
\label{sec:Methodology}
We representing action value using the Bellman equation to select the best actions for each timestep:
\begin{equation}
    V(\boldsymbol{s}) = \max_{\boldsymbol{a} \in \mathcal{A}} Q(\boldsymbol{s}, \boldsymbol{a}), \label{eq:bellman_v}
\end{equation}
\begin{equation}\label{eq:bellman_q}
    Q(\boldsymbol{s},\boldsymbol{a}) = R(\boldsymbol{s},\boldsymbol{a}) + \gamma \sum\nolimits_{u} P(\boldsymbol{u} \mid \boldsymbol{s}, \boldsymbol{a}) \sum\nolimits_{s'} P_G(\boldsymbol{s'} \mid \boldsymbol{u}) V(\boldsymbol{s'}),
\end{equation}

To solve NRMAB problems, we can write down the Bellman equation in Equation \ref{eq:bellman_v} to apply existing RL algorithms like DQN \citep{watkins1992q}. However, the action selection in Equation \ref{eq:bellman_v} is computationally infeasible for large state-action spaces due to exponentially many possible actions.

Our objective is to develop an algorithm capable of consistently selecting near-optimal actions in a scalable manner. To achieve this, we exploit the submodularity of $Q(\boldsymbol{s},\boldsymbol{a})$ to use a \emph{hill-climbing} action selection that only takes $O(n^2)$ time. Because of submodularity, this algorithm yields a $(1-1/e)$-approximation to the true maximum. We refer to this algorithm as Bellman Equation with Hill-Climbing Action Selection.

\begin{algorithm}[H]
\caption{Hill-Climbing Action Selection for the Bellman Equation}
\label{alg:bellman-hillclimb}
\begin{algorithmic}[1]
  \REQUIRE State $s\in\mathcal{S}$, budget $k$, Q-function $Q(\cdot,\cdot)$
  \STATE $A \gets \varnothing$ \COMMENT{initial empty action set}
  \FOR{$j = 1$ \TO $k$}
      \STATE $a^{\star} \gets \displaystyle
             \arg\max_{a \in \mathcal{A} \setminus A}
             Q\!\bigl(s,\,A \cup \{a\}\bigr)$
      \STATE $A \gets A \cup \{a^{\star}\}$
  \ENDFOR
  \RETURN $A$ \COMMENT{greedily constructed size-$k$ action set}
\end{algorithmic}
\end{algorithm}

\subsection{Submodularity of the Q-Function}
Establishing that the state–action value \(Q(\boldsymbol{s},\boldsymbol{a})\) is submodular in the action set \(\boldsymbol{a}\) is pivotal: it unlocks the \(1 - \frac{1}{e}\) performance guarantee of a greedy hill‑climbing strategy that provably avoids the exponential blow‑up of evaluating all \({n \choose k}\) action combinations \citep{kempe2003maximizing}. We proceed with a proof of the submodularity of $Q(\boldsymbol{s},\boldsymbol{a})$.

\begin{restatable}[Submodularity]{theorem}{submodularity}
    Given a submodular value function $V(\boldsymbol{s})$ and a constant state $\boldsymbol{s}$, we show that $Q(\boldsymbol{s},\boldsymbol{a})$ is submodular with respect to $a$.
    \label{thm: submod-q}
\end{restatable}

\begin{proof}[Proof sketch] (Full proof in \ref{appendix:submodularity})

We show that for any \( A \subseteq B \subseteq N \) and \( t \notin B \):
\[
Q(s, A \cup \{t\}) - Q(s, A) \geq Q(s, B \cup \{t\}) - Q(s, B)
\]
In $Q(s,a)$, the reward function \( R(s, a) \) is inherently submodular. We focus on the expected future value component:
\[
\sigma(a) = \sum\nolimits_{s', u} P_G(s' \mid u) P(u \mid s, a) V(s')
\]
We model the state transitions and cascades using coupled probabilistic simulations. For each node \( v \in \mathcal{V} \), we simulate two coin flips: $x_v$, which represents the node's outcome in the transition step under a passive action, and $y_v$, the same node's outcome under an active action. We couple these coinflips such that if $x$ results in activation, the corresponding $y$ must also result in activation. For each edge \( e \in \mathcal{E} \), we simulate a coin flip \( z_e \) to determine if an active node activates its neighbor via a cascade. 
With a full set of coinflips $X,Y,Z$, we deterministically know the set of active nodes after applying an action. Let \(\sigma_{XY}(A)\) denote the set of active nodes after the Transition Step, and \(\sigma_Z(\sigma_{XY}(A))\) denote the set of active nodes after both Transition and Cascade Steps.

We know that \(\sigma_{XY}(A) \subseteq \sigma_{XY}(B)\) for \( A \subseteq B \) and \(\sigma_{XY}(A \cup \{t\}) \setminus \sigma_{XY}(A) = \sigma_{XY}(B \cup \{t\}) \setminus \sigma_{XY}(B)\). Using these properties, the submodularity inequality can be rewritten for the independent cascade as:
\[
\begin{aligned}
\sigma_Z(\sigma_{XY}(A \cup \{t\})) - \sigma_Z(\sigma_{XY}(A))  \\
\geq \sigma_Z(\sigma_{XY}(B \cup \{t\})) - \sigma_Z(\sigma_{XY}(B)).
\end{aligned}
\]

Since we assume that $V(s)$ is submodular with respect to its support set,
and each active node contributes positively to the total weighted node value,
it follows that $V(\sigma_{X,Y,Z}(A))$ is also submodular. Consequently, our expected future value can be formulated as:
\[
\sigma(A) = \sum\nolimits_{X,Y,Z} P(XYZ) \cdot V(\sigma_{X,Y,Z}(A))
\]
which is a non-negative linear combination of submodular functions, maintaining submodularity. Therefore, \( Q(s, a) \), being a sum of submodular functions, is itself submodular.
\end{proof}

This submodularity allows us to utilize a $1-\frac{1}{e}$ optimality guarantee when using hill-climbing algorithm in Algorithm \ref{alg:bellman-hillclimb}, a $O(n^2)$ algorithm scalable to larger problems.

\noindent

\subsection{Proof of Contraction for the Bellman Equation with Hill-Climbing Action Selection}

Given submodularity of $Q(s,a)$, each action set of Bellman equation with hill-climbing action selection has an optimality guarantee of $1-\frac{1}{e}$. We want to show that even under this approximate action selection, the Bellman operator for Bellman equation with hill-climbing action selection (Definition \ref{def:bellman-hillclimb}) is guaranteed to converge.

\begin{definition}[Bellman Operator for Bellman Equation with Hill-Climbing Action Selection]
\label{def:bellman-hillclimb}
Let $\mathcal{S}$ be the set of global states, and let $\mathcal{A}$ be the set of all \emph{single} actions (e.g., single nodes) that can be targeted at a given timestep. Suppose we want to build a final action set of size $k$ from $\mathcal{A}$, subject to a budget of $k$.

\noindent
We define the Bellman Operator for this variation of the Bellman equation as $B$,

\begin{equation}
\label{eq:bellman-hc-update}
B V(s) \;=\; \max_{a^{\text{hc}}\in \mathcal{A}}\{R(s,a^\text{hc})+ \gamma
\sum\nolimits_{u\in\mathcal{S}} P\!\bigl(u\mid s,a^{\mathrm{hc}}(s)\bigr)
\sum\nolimits_{s'\in\mathcal{S}} P_G\!\bigl(s'\mid u\bigr)\,V(s')\},
\end{equation}
where \(a^{hc}\) is the output of Algorithm \ref{alg:bellman-hillclimb} and \(Q\) is given in
Equation (\ref{eq:bellman_q}).
\end{definition}

Based on Theorem 4.3 in \citep{nemhauser1978analysis} and the submodularity given by Theorem~\ref{thm: submod-q}, we can show that the greedy algorithm in Algorithm~\ref{alg:bellman-hillclimb} discovers a set of actions that yields a $V(s)$ at least
\(1-1/e\) of the true maximum. 

However, because Algorithm \ref{alg:bellman-hillclimb} does not always yield the optimal action, the traditional Banach Fixed-Point argument for Bellman Operator contraction does not work (see Appendix \ref{appendix: failure}), and we instead construct an alternative approach exploiting the structure of repeated hill‑climbing updates. 

\begin{restatable}[Contraction]{theorem}{contraction}
    Bellman Operator for Bellman Equation With Hill-Climbing Action Selection ($B$) is a $\gamma$ contraction under the supremum norm $\|\cdot\|_\infty$.
    \label{thm: contract}
\end{restatable}

\begin{proof}[Proof Sketch]

Our proof centers around redefining $B$ as a multi-bellman operator as defined by \citet{carvalho2023multibellmanoperatorconvergenceqlearning}.

\paragraph{Intuition} Because the Bellman equation with hill climbing action selection finds the approximate rather than the best set of actions at each timestep, the traditional proof for contraction does not work. Thus, we deconstruct our Bellman Operator for Hill-Climbing Action Selection (definition \ref{def:bellman-hillclimb}) into a multi-bellman operator $\widetilde{B}$ as defined by \citep{carvalho2023multibellmanoperatorconvergenceqlearning}. Each application of this operator selects the single next best action to take given a state and partial action set, effectively mimicking one step in the hill-climbing algorithm. Applying this $k$ times becomes equivalent to one application of $B$. We prove this equivalence, and borrowing the proof of convergence for the multi-bellman operator, prove that $B$ thus converges.

\begin{definition}[Multi-Bellman Operator for the Hill-Climbing Variant]
\label{def:new_bell_op}
We recast our incremental set-building procedure as follows.  
Let each ``meta-state'' be denoted by $\widetilde{s} = (s,A, t)$, 
where $s$ is the environment state, 
$A \subseteq \mathcal{A}$ is the set of actions selected so far, and $t$ is the current timestep. Additionally, define $\widetilde{\gamma}^k = \gamma$.
Write $\widetilde{s}_0 = \bigl(s,\varnothing,0\bigr),\,\widetilde{s}_1=\bigl(s,\{a_0\}, 1\bigr),\,\dots,\widetilde{s}_k=\bigl(s,\{a_0,\dots,a_{k-1}\},k\bigr).$
Define the modified reward
\[
  \widetilde{R}\bigl(\widetilde{s}, a \bigr) 
  =  
  \frac{1}{\widetilde{\gamma}^t}(R\left(s, A \cup \{a\} \right) - R(s, A)).
\]

Then for $t<k$, picking a \emph{single} new action $a$ corresponds to moving from 
$\widetilde{s}_j = (s,A, t)$ 
to 
$\widetilde{s}_{j+1} = \bigl(s,A \cup \{a\}, t+1\bigr)$,  
and we define 
\[
(\widetilde{B}V)\bigl(\widetilde{s}_j\bigr)
~
=\;
\max_{\,a\in\mathcal{A} \setminus A}
\Bigl\{
  \widetilde{R}\bigl(\widetilde{s}_j,\,a\bigr)
  \;+\;
  \widetilde{\gamma}\,V\bigl(\widetilde{s}_{j+1}\bigr)
\Bigr\}.
\]

\noindent
\textbf{At $t = k$:}  
the action set $A$ is fully chosen (i.e.\ $\widetilde{s}_{k} = (s,\{a_0,\dots,a_{k-1}\}, k)$). 
One more application of $\widetilde{B}$ (when $t=k$) then \emph{applies} these actions to $s$, 
causing a transition to $s'$.
\[
  (\widetilde{B}V)(s,A,k)
  \;=\;
     \widetilde{\gamma}\,
     \mathbb{E}_{s'}\!\bigl[\,V\!\bigl(s',\varnothing,0\bigr)\bigr]
\]
Hence, $\widetilde{B}$ captures both the step-by-step incremental selection of actions for $t<k$, and the final transition applying the chosen set $A$ when $t=k$.
\end{definition}
After defining $\widetilde{B}$, we leverage it to construct a MDP $\mathcal{M}_{hc}$ for this Bellman operator -- one that is equivalent to the MDP for $B$.
\begin{restatable}[Hill-Climbing Equivalence]{theorem}{equivalence}    
    Applying $k$ iterations of Multi-Bellman Operator for Hill-Climbing Variant (Definition \ref{def:new_bell_op}) is equivalent to one application of Bellman Operator for Bellman equation with hill climbing action selection (Definition \ref{def:bellman-hillclimb}) with an action budget of $k$.
    \label{thm:equiv}
\end{restatable} 
\begin{proof}[Proof Sketch] (Full proof in Appendix \ref{appendix:equiv})
Apply $\widetilde B$ exactly $k$ times.  Each
step adds a \emph{marginal} reward
$\widetilde R(\widetilde s_t,a_t)
=\widetilde\gamma^{-(k-t)}
\bigl(R(s,A_t\!\cup\!\{a_t\})-R(s,A_t)\bigr)$,
so the discounted sum telescopes:
\[
\sum\nolimits_{t=0}^{k-1}\widetilde\gamma^{\,t}\widetilde R(\widetilde s_t,a_t)
=R\bigl(s,\{a_0,\dots,a_{k-1}\}\bigr).
\]
After the $k$‑th pick the augmented state resets to the environment
state $s'$, and because $\widetilde\gamma^{\,k}=\gamma$ the future value
is $\gamma V(s')$.  Hence
\[
(\widetilde B^{\,k}V)(s)=R\bigl(s,\{a_0,\dots,a_{k-1}\}\bigr)+\gamma V(s')
=(BV)(s).
\]
Therefore $\,\widetilde B^{\,k}\,$ coincides with the standard Bellman
update, so all usual contraction and
convergence results carry over. This concludes the proof of Theorem \ref{thm:equiv}.
\end{proof}

Given the equivalence shown in Theorem \ref{thm:equiv}, we can directly apply the proof of Lemma 1 in \cite{carvalho2023multibellmanoperatorconvergenceqlearning} to $\widetilde{B}$. By following their proof, we find that $\widetilde{B}^k$ is a $\widetilde{\gamma}^k$ contraction (for the full process see Appendix \ref{appendix:contraction}). Because $\widetilde{\gamma}^k=\gamma$ and $\widetilde{B}^k=B$, $B$ is a $\gamma$-contraction. This concludes the proof of Theorem \ref{thm: contract}.
\end{proof}
Thus, we have show that even though greedy hill-climbing action selection only provides an action $1-\frac{1}{e}$ of the optimum, Bellman equation using hill-climbing action selection is still a $\gamma$-contraction. 
\subsection{Deep Q-Learning with Hill-Climbing}
To solve a NRMAB problem, we propose a Deep Q-Network using hill-climbing. Similar to a traditional DQN, our neural network takes in a representation of the state and action and pass it through three fully connected hidden layers to producing a Q-value for each state-action pair \( Q(\boldsymbol{s}, \boldsymbol{a}) \). To scale this to large action spaces, we iterate through the list of all possible single actions and utilize a neural network to predict the Q value of each single action $a$. Then, we greedily select the $k$ actions with the highest Q-values. This leverages the submodular properties of the Bellman equation to achieve the \(1 - \frac{1}{e}\) performance guarantee. We combine DQN and hill-climbing action selection to design a scalable Q-learning algorithm (Algorithm \ref{alg:dqn_tianshou_refined}) to solve NRMAB problems.

\begin{algorithm}[H]
\caption{Hill-Climbing DQN}
\label{alg:dqn_tianshou_refined}
\begin{algorithmic}[1]
\STATE \textbf{Initialization:} 
Neural network $Q(s,a; \theta)$, replay buffer
\WHILE{until $\theta$ converges}

\STATE \textbf{Hill-climbing action:} intervention set $A = \emptyset$.
    \WHILE{$|A| < k$ (budget for intervention)}
    \STATE Solve $v^* = \arg\max_{v \in V} Q(s, 1_{A\cup \{v\}})$
    \STATE Update $A \gets A \cup \{ v \}$
    \ENDWHILE
    \STATE Execute $a = 1_A$ and collect experience 
\STATE \textbf{DQN Updates:} Sample mini-batches from the replay buffer and run gradient descent to update $\theta$.
\ENDWHILE
\STATE \textbf{Output:} Q network parameter $\theta$
\end{algorithmic}
\end{algorithm}

\paragraph{Graph Neural Network Optimization} In order to better account for network effects, we optimize this approach by implementing a graph neural network in addition to a simple DQN to leverage relational dependencies within the network. We maintain all other properties for the GNN, including using hill-climbing action selection. 

\section{Experiments}
\label{sec:Experiments}
\subsection{Domain}
The motivating application is health‑care intervention planning, where limited resources (e.g., vaccinations, diagnostic tests, treatment slots, adherence reminders) must be allocated over time while infection or non‑adherence spreads through a contact network. Classical RMAB models treat patients (arms) as independent, yet in epidemiology and behavioral health network spill‑overs can impact outcomes. Our experiments therefore contrast the effectiveness of network‑blind approaches with our network-aware algorithm developed for the NRMAB model.
\subsection{Simulation}
We evaluate our algorithms on a real network collected from a village in India through household surveying \cite{data}, augmenting it with synthetic node attributes and cascade probabilities. The data is given as an edgelist where each edge represents real world contact. The data depicts a multigraph, but we remove redundant edges to create a simple graph. After processing, the network contains 202 nodes and 692 edges. We use the edgelist to build our graph, then randomly generate attributes for each node in the graph and set the cascade probability. The results of the comparison are shown in Figure \ref{fig:GNN_Comp}.

The DQN is independently defined and trained using TianShou and PyTorch for the neural network, and Gymnasium for the simulation environment. GNN uses PyTorch, PyTorch Geometric, and Gymnasium. DQN is trained over seven epochs of 1000 steps, and GNN is trained for 100 episodes. Higher training times show minimal improvement.

After training, each algorithm is evaluated using 10 random seeds, with 50 simulations per seed, each running for 30 timesteps. We collect mean cumulative reward, mean reward per timestep, mean activation percentage from each seed. All experiments can be run locally on a single RTX 3050Ti GPU in < 12 hours.

\subsection{Real World Environment}
The India contact graph (\(n = 202,\; |E| = 692\)) is adopted as a high‑fidelity synthetic contact network: each person is an individual arm and
every edge carries a fixed cascade probability \(w_{vw} = 0.03\),
representing the chance that health resources or infections propagate between close contacts. Empirical work shows that such digitally inferred graphs capture the dominant pathways of disease and behavioral diffusion in real populations~\citep{salathe2010high}.

We initialize the system with no active nodes and impose an intervention budget of \(k = 30\) actions per timestep, mimicking limited daily vaccine or test capacity. Policies are evaluated against other intervention and no‑intervention baselines so that cumulative‑reward gains translate directly into expected infections or adverse events averted.
\subsection{Baseline Algorithms}
To evaluate the effectiveness of our DQN and GNN algorithms, we compare them with three other algorithms. Tabular Q-learning solves the full Bellman equation for each state-action pair for small state-action sizes. 1-Step Look-Ahead performs hill-climbing by calculating the value of activating a node in a state and taking into account network effect but ignoring future states. Whittle Index is a traditionally optimal method for solving RMABs without considering network effects. 
\section{Results \& Discussion}
\label{sec:Results}
\subsection{Performance on Real-World Graph}
Figure \ref{fig:GNN_Comp} shows the average percentage of activated nodes over 30 timesteps on the India contact network, aggregated across simulations on 10 random seeds. The GNN-based policy consistently achieves the highest activation, converging to over 82\% by timestep 30. DQN and Whittle follow closely, leveling off around 80–81\%, while the 1-step lookahead trails slightly behind. In contrast, the no-intervention baseline stabilizes under 71\%, highlighting the effectiveness of all intervention strategies relative to doing nothing.

These results suggest that explicitly accounting for network structure can substantially improve the reach of health interventions over time. In practical terms, this means more individuals are consistently reached and maintained in a healthy state, even when resources are limited. Though in this experiment the performance gap relative to naive baselines is relatively small, the significant number and diversity of trials run suggest a statistically significant improvement. More complex scenarios should significantly increase the performance disparity in favor of GNN and DQN algorithms.

\begin{figure*}[h!]
    \centering
    \includegraphics[width=0.8\linewidth]{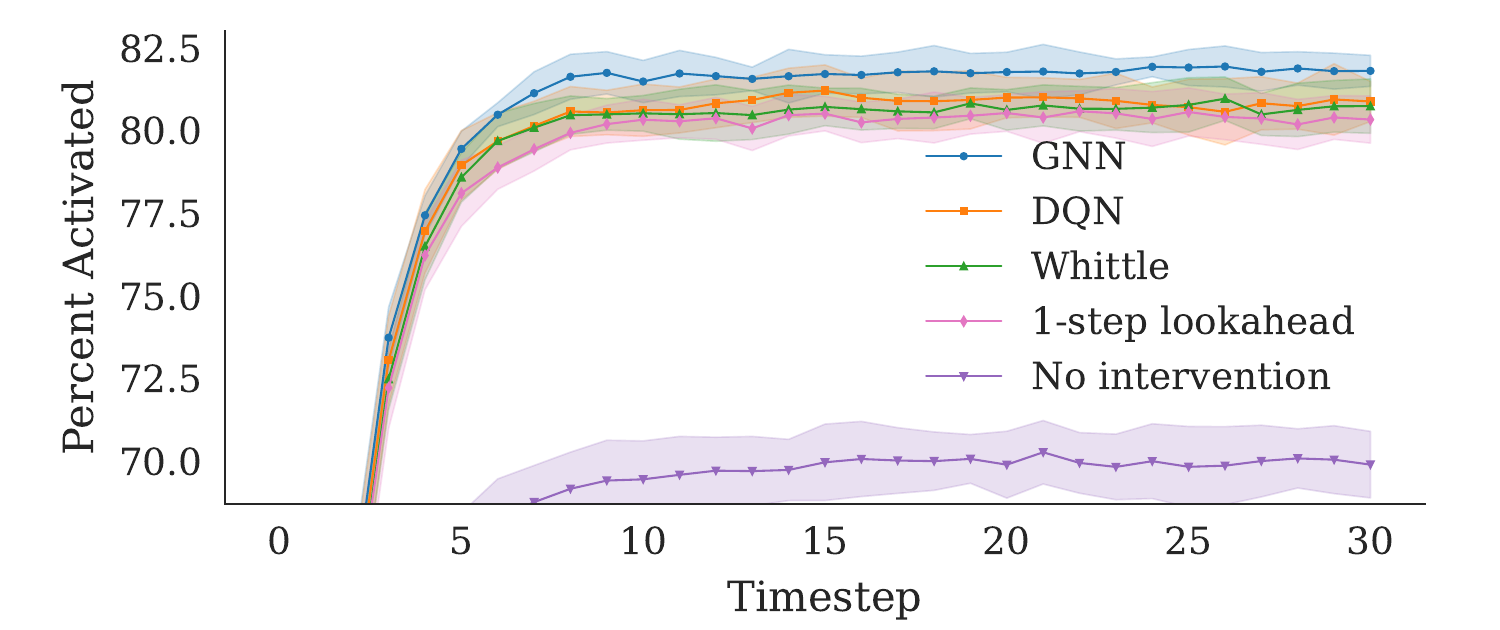}
    \caption{Mean\,$\pm$\,SD fraction of activated nodes over 30 timesteps on the India contact network (\(n=202,\ |E|=692;\ k=20;\ 10\ \text{seeds}\times 50\ \text{runs}\)) shows the GNN stabilizing near \(82\%\) activation and consistently outperforming DQN, Whittle index, 1‑step look‑ahead, and the no‑intervention baseline.}
\label{fig:GNN_Comp}
\end{figure*}
\subsection{Optimality Verification}
Figure \ref{fig:tabular} shows the performance of DQN with hill-climbing compared to Tabular Q-learning, an approach that gives near-optimal solutions at every timestep. We validate the optimality guarantee delivered by the submodular greedy algorithm as shown in Theorem \ref{thm: submod-q}: DQN performs at a very similar level to Tabular Q-learning. The extreme similarity in performance may be due to the small graph size tested, as the runtime of Tabular Q Learning rapidly explodes on larger graphs.
\subsection{Computational Cost}
Figure \ref{fig:Difference_300_nodes} shows the runtime difference between GNN, DQN with hill-climbing, and Tabular Q-learning. Tabular Q-learning runtime increases exponentially with increasing nodes, while GNN and DQN with hill-climbing increases about linearly. This aligns with theoretical runtime benefits in Algorithm \ref{alg:bellman-hillclimb} while achieving comparative performance to the optimal algorithm(see Figure \ref{fig:tabular}).

\begin{figure*}[h!]
    \centering
    \begin{minipage}{.48\textwidth}
        \centering
        \includegraphics[width=\linewidth]{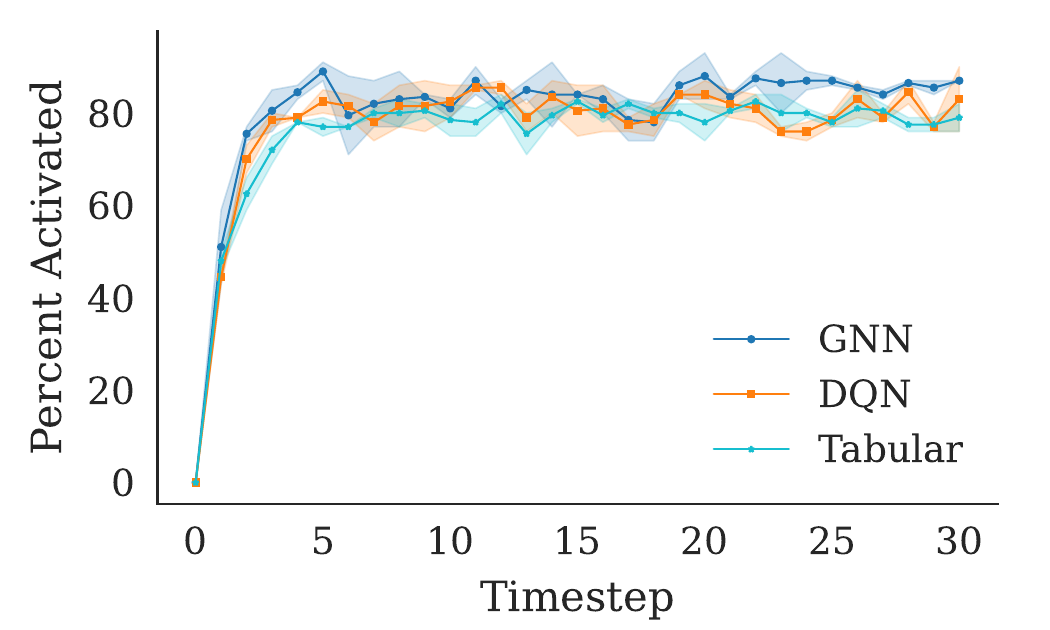}
        \caption{Mean\,$\pm$\,SD activation fraction over 30 timesteps on a 10‑node graph. DQN and GNN match tabular Q‑learning’s near‑optimal performance in networked RMABs.}
    \label{fig:tabular} 
    \end{minipage}
    \hfill
    \begin{minipage}{.48\textwidth}
        \centering
        \includegraphics[width=\linewidth]{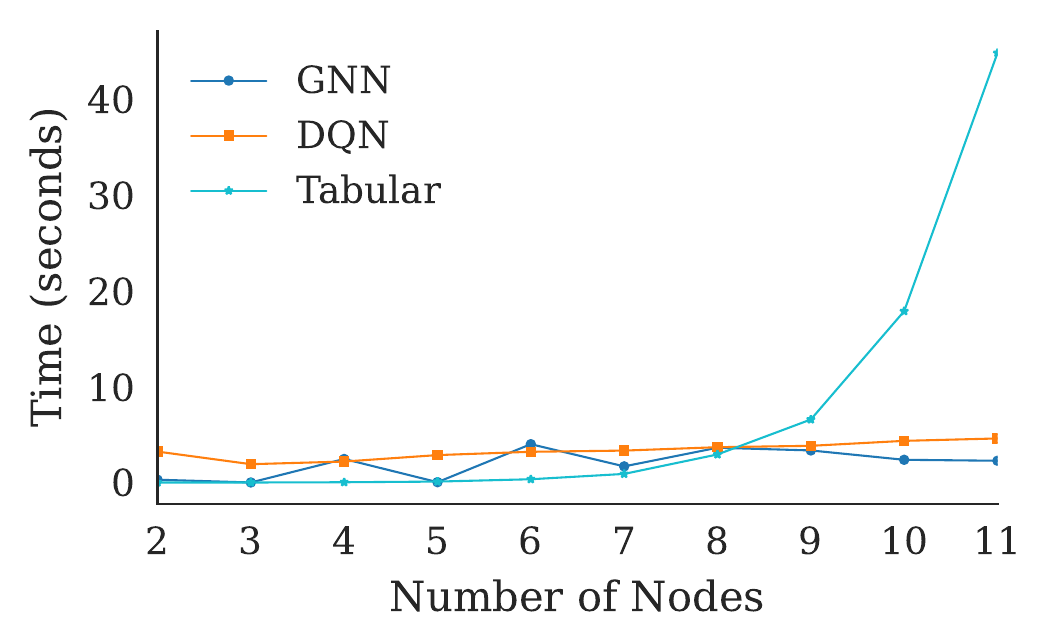}
        \caption{Total runtime (per epoch runtime for DQN and GNN) versus graph size $n$. Results reveal tabular Q‑learning’s exponential run-time growth, while DQN and GNN grow linearly.}
    \label{fig:Difference_300_nodes}
    \end{minipage}
\end{figure*}
\section{Conclusion \& Future Work}
\label{sec:conclusion}
We introduced the \emph{NRMAB} model to capture network spill‑overs in sequential resource allocation, proved its Bellman operator is both sub\-modular and a
$\gamma$‑contraction -- so a greedy hill‑climbing policy enjoys a $(1-1/e)$ guarantee -- and demonstrated a GNN implementation that outperforms strong baselines on contact‑network data. Current experiments still rely on synthetic node attributes and fixed cascade
rates; transforming real world observations into node attributes and cascade weights is pivotal to further experimentation. Current assumptions of full observability and static cascade probabilities can also be relaxed for broader applicability.

In future work, collecting data tailored to NRMAB enables experiments that bridge theory and practice. Partial observability modeled via a belief‑state (POMDP) NRMAB can align the framework with real-world deployments. Allowing edge‑specific cascade probabilities that evolve over time can capture changing behavior. Finally, our results hint at fairness–efficiency trade‑offs when node values vary widely; embedding fairness constraints directly into the hill‑climbing step could yield more socially responsible policies.

\bibliographystyle{abbrvnat}   

\bibliography{main}

\newpage
\appendix
\section{Technical Appendices and Supplementary Material}

\subsection{Full Submodularity Proof}
\label{appendix:submodularity}
\submodularity*
\begin{proof}[Proof]
We demonstrate that our implementation of the Bellman equation exhibits submodular
properties, which are crucial for the efficiency and effectiveness of greedy algorithms.
Submodularity ensures that the marginal gain of adding an element to a set decreases
as the set grows, a property leveraged in influence maximization.

Formally, we aim to show that for any $A \subseteq B \subseteq N$ and $t \notin B$:
\begin{equation}
    Q(s, A \cup \{t\}) - Q(s, A)
    \;\ge\;
    Q(s, B \cup \{t\}) - Q(s, B),
\end{equation}
where $Q(s, A)$ represents the value of taking action set $A$ in state $s$.

The reward function $R(s,a)$ is inherently submodular. We focus on the expected
future value component:
\begin{equation}
    \sigma(a) = \sum_{s',u} P_G(s' \mid u)\,P(u \mid s, a)\,V(s').
\end{equation}

To analyze $\sigma(a)$, we model the state transitions using coupled probabilistic
simulations. Specifically, for each node $v \in \mathcal{V}$, we simulate two coin flips:
\begin{itemize}
    \item $x_v$: outcome under passive action with probability $P(s = s_0, a = 0, u = u_1)$;
    \item $y_v$: outcome under active action with probability $P(s = s_0, a = 1, u = u_1)$.
\end{itemize}
We couple these coin flips such that if $x_v$ results in activation ($u = 1$), then
$y_v$ also results in activation. This coupling reflects the assumption that active
actions have transition probabilities at least as good as passive actions, ensuring:
\begin{align}
    P(s = 0, a = 1, u = 1) &\ge P(s = 0, a = 0, u = 1), \\
    P(s = 1, a = 1, u = 1) &\ge P(s = 1, a = 0, u = 1).
\end{align}

Additionally, for each edge $e \in \mathcal{E}$, we simulate a coin flip $z_e$ with
bias $p_{v,w}$ to determine if an active node activates its neighbor via a cascade.
A heads outcome denotes an active edge, leading to activation, while tails denote
no activation.

Through these coupled simulations, we can deterministically determine the set of
active nodes after applying an action. Let $\sigma_{X,Y}(A)$ denote the set of active
nodes after the Transition Step, and let $\sigma_Z(\sigma_{X,Y}(A))$ denote the set
of active nodes after both Transition and Cascade Steps. We establish the following
properties:
\begin{enumerate}
    \item $\sigma_{X,Y}(A) \subseteq \sigma_{X,Y}(B)$ for $A \subseteq B$. This is because adding more actions (from $A$ to $B$) cannot decrease the set of active nodes due to the coupling of $x_v$ and $y_v$.
    \item $\sigma_{X,Y}(A \cup \{t\}) \setminus \sigma_{X,Y}(A) = \sigma_{X,Y}(B \cup \{t\}) \setminus \sigma_{X,Y}(B) = v'$, where $v'$ represents the newly activated nodes resulting from adding action $t$. This holds because the additional action $t$ affects nodes in the same manner regardless of the existing set $A$ or $B$, thanks to the coupling ensuring $y_v \ge x_v$.
\end{enumerate}

Using these properties, the submodularity inequality can be rewritten for the
independent cascade as:
\begin{equation}
    \sigma_Z\big(\sigma_{X,Y}(A \cup \{t\})\big)
      - \sigma_Z\big(\sigma_{X,Y}(A)\big)
    \;\ge\;
    \sigma_Z\big(\sigma_{X,Y}(B \cup \{t\})\big)
      - \sigma_Z\big(\sigma_{X,Y}(B)\big).
\end{equation}
This inequality demonstrates that the number of active nodes after an action is
submodular with respect to the size of the action set.

Since we assume $V(s)$ is submodular with respect to its support set, and each active
node contributes positively to the total weighted node value, it follows that
$V(\sigma_{X,Y,Z}(A))$ is also submodular. Consequently, our expected future value
can be formulated as:
\begin{equation}
    \sigma(A) = \sum_{X,Y,Z} P(XYZ)\, V\big(\sigma_{X,Y,Z}(A)\big).
\end{equation}
This is a non-negative linear combination of submodular functions, maintaining
submodularity. Therefore, $Q(s,a)$, being a sum of submodular functions, is itself
submodular.
\end{proof}

\subsection{Failure in Traditional Proof for Bellman Equation Convergence}
\label{appendix: failure}

For any two value functions
$V,W : \mathcal{S}\!\to\!\mathbb{R}$, let
\[
  (HV)(s)=\max_{a\in\mathcal{A}}
     \Bigl\{ R(s,a)+\gamma\,\mathbb{E}[V(S')\mid s,a] \Bigr\},
\]
and define $HW$ analogously.  
Choose $a^{\star}(s)\in\arg\max_{a}Q_V(s,a)$, where
$Q_V(s,a)=R(s,a)+\gamma\mathbb{E}[V(S')\mid s,a]$.
Then
\[
  \begin{aligned}
  |(HV)(s)-(HW)(s)|
    &\;=\;
      \Bigl| Q_V\bigl(s,a^{\star}(s)\bigr)
            -\max_{a}Q_W(s,a) \Bigr|                        \\[2pt]
    &\le
      \Bigl| Q_V\bigl(s,a^{\star}(s)\bigr)
            -Q_W\bigl(s,a^{\star}(s)\bigr) \Bigr|            \\[2pt]
    &\le \gamma\,\|V-W\|_\infty,
  \end{aligned}
\]
and taking the supremum over $s$ yields the $\gamma$‑contraction.

The proof relies on re‑using the \emph{same} action
$a^{\star}(s)$ under both $V$ and $W$.
Algorithm \ref{alg:bellman-hillclimb}, however, returns
$\widetilde{a}(s,V)$ that is merely near‑optimal for $V$; when $W$
differs from $V$, the algorithm may choose an entirely different
$\widetilde{a}(s,W)$.
Consequently we can bound only
\[
  \bigl|( \widetilde{H}V)(s) - (\widetilde{H}W)(s)\bigr|
  \;\le\;
  \underbrace{
    \bigl| Q_V\bigl(s,\widetilde{a}(s,V)\bigr)
          -Q_W\bigl(s,\widetilde{a}(s,V)\bigr)
    \bigr|}_{\le\gamma\|V-W\|_\infty}
  \;+\;
  \underbrace{
    \bigl| Q_W\bigl(s,\widetilde{a}(s,V)\bigr)
          -Q_W\bigl(s,\widetilde{a}(s,W)\bigr)
    \bigr|}_{\text{\small loss from approximate actions}},
\]
and the second term has \emph{no} $\gamma$ factor.  
Hence $\widetilde{H}$ need not be a contraction, and classical
Banach‑fixed‑point arguments fail.  
The convergence analysis in Section 4.2
circumvents this obstacle by treating $B$ as a
Multi‑Bellman operator and exploiting the structure of repeated
hill‑climbing updates instead of relying on the traditional contraction argument.

\subsection{Expanded Proof of Contraction from \citet{carvalho2023multibellmanoperatorconvergenceqlearning}}
\label{appendix:contraction}
\contraction*
\begin{proof}
    Recall from Theorem \ref{thm:equiv} that $\widetilde{\mathbf{B}}^k$ 
is equivalent to the $k$-fold composition of a single-step operator $\widetilde{\mathbf{B}}$, 
where the ``state'' is the meta-state $\widetilde{s} = (s,A, t)$, and we add one action at a time.

\paragraph{Single-step contraction.} 
Take any two $Q$-functions, $Q_1$ and $Q_2$. We compute:
\[
\|\widetilde{\mathbf{B}}Q_1 - \widetilde{\mathbf{B}}Q_2\|_\infty 
~=\; 
\max_{(\widetilde{s}_0,\widetilde{a}_0)} 
\Bigl|\,
   \bigl[\widetilde{R}(\widetilde{s}_0,a_0) + \widetilde{\gamma} \max_{\,a_1}Q_1(\widetilde{s}_1,a_1)\bigr]
   ~-~
   \bigl[\widetilde{R}(\widetilde{s}_0,a_0) + \widetilde{\gamma} \max_{\,a_1}Q_2(\widetilde{s}_1,a_1)\bigr]
\Bigr|.
\]
Since $\widetilde{R}(\widetilde{s}_0,a_0)$ cancels, we get
\[
\|\widetilde{\mathbf{B}}Q_1 - \widetilde{\mathbf{B}}Q_2\|_\infty
~=~
\widetilde{\gamma} 
\max_{(\widetilde{s}_0,a_0)}
\Bigl|\,
   \max_{\,a_1}Q_1(\widetilde{s}_1,a_1) 
   ~-~
   \max_{\,a_1}Q_2(\widetilde{s}_1,a_1)
\Bigr|.
\]
Using the standard inequality $\bigl|\max_x f(x) - \max_x g(x)\bigr| \,\le\, \max_x \bigl|f(x)-g(x)\bigr|$, we obtain
\[
\|\widetilde{\mathbf{B}}Q_1 - \widetilde{\mathbf{B}}Q_2\|_\infty 
~\le~
\widetilde{\gamma}
\max_{(\widetilde{s}_0,a_0)} \bigl|Q_1(\widetilde{s}_1,a_1) - Q_2(\widetilde{s}_1,a_1)\bigr|
~=~
\widetilde{\gamma} \,\|Q_1 - Q_2\|_\infty.
\]
Hence $\widetilde{\mathbf{B}}$ is indeed a $\widetilde{\gamma}$-contraction under the supremum norm.

\paragraph{Composition into $(\widetilde{\mathbf{B}})^k$.}
By definition,
\[
(\widetilde{\mathbf{B}})^k 
~\;=\;
\underbrace{\widetilde{\mathbf{B}}
~\circ~
\widetilde{\mathbf{B}}
~\circ~
\cdots
~\circ~
\widetilde{\mathbf{B}}}_%
{k\text{ times}}.
\]
To show $(\widetilde{\mathbf{B}})^k$ is a $\gamma$-contraction, we proceed by induction on $k$:
\begin{itemize}
    \item For $k=1$, we have just shown $\widetilde{\mathbf{B}}$ itself contracts by factor~$\widetilde{\gamma}$.
    \item Assume $(\widetilde{\mathbf{B}})^k$ is a $\widetilde{\gamma}^k$-contraction. Then
    \[
    \bigl\|(\widetilde{\mathbf{B}})^{k+1} Q_1 
           ~-~
           (\widetilde{\mathbf{B}})^{k+1} Q_2
    \bigr\|_\infty
    ~=~
    \bigl\|\widetilde{\mathbf{B}}\bigl[(\widetilde{\mathbf{B}})^kQ_1\bigr] 
           ~-~
           \widetilde{\mathbf{B}}\bigl[(\widetilde{\mathbf{B}})^kQ_2\bigr]
    \bigr\|_\infty
    ~\le~
    \widetilde{\gamma}\,\|(\widetilde{\mathbf{B}})^kQ_1 - (\widetilde{\mathbf{B}})^kQ_2\|_\infty
    \]
    by the single-step contraction. Applying the induction hypothesis,
    \[
    \|(\widetilde{\mathbf{B}})^kQ_1 - (\widetilde{\mathbf{B}})^kQ_2\|_\infty
    ~\le~
    \widetilde{\gamma}^k\,
    \|Q_1 - Q_2\|_\infty.
    \]
    Consequently,
    \[
    \bigl\|(\widetilde{\mathbf{B}})^{k+1} Q_1 
           ~-~
           (\widetilde{\mathbf{B}})^{k+1} Q_2
    \bigr\|_\infty
    ~\le~
    \widetilde{\gamma} \,\widetilde{\gamma}^k\,
    \|Q_1 - Q_2\|_\infty
    ~=~
    \widetilde{\gamma}^{k+1}\,\|Q_1 - Q_2\|_\infty.
    \]
\end{itemize}
Thus by induction, $(\widetilde{\mathbf{B}})^k$ is a $\widetilde{\gamma}^k$-contraction. By definition \ref{def:new_bell_op} $\widetilde{\gamma}^k=\gamma$, thus $\widetilde{B}^k$ is a $\gamma$ contraction under the supremum norm $\|\cdot\|_\infty$.

\paragraph{Conclusion.} 
Since our ``Bellman equation With hill-climbing action selection'' is equivalent to $(\widetilde{\mathbf{B}})^k$, we conclude it is a $\gamma$-contraction in the sup norm. Hence, like the standard multi-Bellman operator of \cite{carvalho2023multibellmanoperatorconvergenceqlearning}, it converges to a unique fixed point under $0\le \gamma<1$ and bounded rewards.
\end{proof}
\subsection{Full Proof of Equivalence for Definition \ref{def:bellman-hillclimb} and \ref{def:new_bell_op}} 
\label{appendix:equiv}
\equivalence*
    We introduce the concept of a Multi-Bellman Operator defined by \cite{carvalho2023multibellmanoperatorconvergenceqlearning}. Given Bellman operator $H$ and $Q$ function $q$, a Multi-Bellman operator is defined as:

\[
(\mathbf{H}^n q)(x_0,a_0)
~=~
\mathbb{E}\Bigl[
  r(x_0,a_0)
  ~+~
  \gamma\,\max_{a_1\in\mathcal{A}}
    \mathbb{E}\Bigl[
      r(x_1,a_1)
      ~+~
      \gamma\,\max_{a_2\in\mathcal{A}}
        \mathbb{E}\Bigl[
          \cdots
          ~+~
          \gamma\,\max_{a_n\in\mathcal{A}}\,q\bigl(x_n,a_n\bigr)
        \Bigr]
    \Bigr]
\Bigr].
\]

From definition \ref{def:new_bell_op}, we can reduce the algorithm from definition \ref{def:bellman-hillclimb} into a Multi-Bellman operator as such:
\[
(\widetilde{\mathbf{B}}^k Q)(s,a)
~=~
\mathbb{E}\Bigl[
  \widetilde{R}(\widetilde{s_0},a_0)
  ~+~
  \gamma\,\max_{a_1\in\mathcal{A}}
    \mathbb{E}\Bigl[
      \widetilde{R}(\widetilde{s_1}, a_1)
      ~+~
      \gamma\,\max_{a_2\in\mathcal{A}}
        \mathbb{E}\Bigl[
          \cdots
          ~+~
          \gamma\,\max_{a_{k-1}\in\mathcal{A}}\,Q\bigl(\widetilde{s}_{k-1}, a_{k-1}\bigr)
        \Bigr]
    \Bigr]
\Bigr].
\]

In order words, we apply Bellman Operator for Hill-Climbing Action Selection once per action until we reach state $\widetilde{s}_{k-1}$, which represents $(s,\{a_0,...,a_k\}, k-1)$. We can then apply the actions in $\widetilde{s}_k$ to $s$ to transition into state $s'$. 
\begin{proof}
    We proceed by showing that applying $k$ iterations of the Bellman Operator for Hill-Climbing Variant is equivalent to one iteration of the Bellman equation with Hill-Climbing Action Selection, which is the value $R(s, \{a_0, a_1,...,a_{k-1}\}) + \gamma V(s')$.
\[
    \widetilde{B}^kV(s) = \max_{a_0 \in \mathcal{A}}\Bigl[\widetilde{R}(\widetilde{s}_0, a_0) + \widetilde{\gamma}\Bigl[\max_{a_1 \in \mathcal{A}\setminus A }(\widetilde{R}(\widetilde{s}_1, a_1) + \widetilde{\gamma}\Bigl[...+\widetilde{\gamma}\max_{a_k \in \mathcal{A}\setminus A }(\widetilde{R}(\widetilde{s}_{k-1}, a_{k-1}))+\widetilde{\gamma}V(\widetilde{s}_k) \Bigl]\Bigl]\Bigl]
\]
    Expanding this and simplifying, we get:
\[
    \widetilde{B}^kV(s) = (R(s, \{a_0\}) - R(s, \varnothing)) + \widetilde{\gamma}\Bigl[\frac{1}{\widetilde{\gamma}}(R(s, \{a_0, a_1\}) - R(s, \{a_0\})\Bigl] + \widetilde{\gamma}^2\Bigl[\frac{1}{\widetilde{\gamma}^2}(R(s, \{a_0, a_1, a_2\}) - R(s, \{a_0, a_1\})\Bigr] 
\]
\[
    + ...+ \widetilde{\gamma}^{k-1}\Bigl[\frac{1}{\widetilde{\gamma}^{k-1}}(R(s, \{a_0,...,a_{k-1}\}) - R(s, \{a_0,...,a_{k-2}\}) + \widetilde{\gamma}^kV(\widetilde{s}_k)\Bigr] 
\]

Notice that $R(s, \{a_0\}) - \widetilde{\lambda}\frac{1}{\lambda}(R(s,\{a_0\}) = 0$, $\widetilde{\lambda}\frac{1}{\widetilde{\lambda}}R(s, \{a_0, a_1\}) - \widetilde{\lambda}^2\frac{1}{\widetilde{\lambda}^2}R(s, \{a_0, a_1\}) = 0$, and so on. Thus, our $\widetilde{B}^kV(s)$ simplifies to 
\[
 \widetilde{B}^kV(s) = -R(s,\widetilde{\varnothing}) + R(s, \{a_0,...,a_{k-1}\})
\]
Since $R(s,\varnothing) = 0$ and $\widetilde{\gamma}^k = \gamma$, we ultimately arrive at 
\[ 
\widetilde{B}^kV(s) = R(s, \{a_0,...,a_{k-1}\}) + \gamma V(\widetilde{s_k})
\]
However, at $s_k$, $t=k$ and by Definition \ref{def:new_bell_op} we transition our state into $\widetilde{s}'=(s',\varnothing,0)$ which can be simplified to $s'$. So, with an additional application of $\widetilde{B}$, our final formulation becomes
\[ 
\widetilde{B}^{k}V(s) = R(s, \{a_0,...,a_{k-1}\}) + \gamma V(s')
\]

Note that by Definition \ref{def:new_bell_op}, we have selected the exact same actions as we would have using Bellman equation with Hill-Climbing Action Selection, and the final value of $\widetilde{B}^kV(s)$ is equivalent to that of $BV(s)$. This concludes our proof.
\end{proof}

\end{document}